\documentclass{article}
\usepackage{ijcai22}

\usepackage{times}
\usepackage{soul}
\usepackage{url}
\usepackage[hidelinks]{hyperref}
\usepackage[utf8]{inputenc}
\usepackage[small]{caption}
\usepackage{graphicx}
\usepackage{amsmath}
\usepackage{amsthm}
\usepackage{booktabs}
\usepackage{algorithm}
\usepackage{algorithmic}
\newtheorem{theorem}{Theorem}

\title{Representation and Synthesis of C++ Programs for Generalized Planning}

\author{
Javier Segovia-Aguas$^1$
\and
Yolanda E-Mart\'in$^2$\And
Sergio Jim\'enez$^{3}$
\affiliations
$^1$Universitat Pompeu Fabra\\
$^2$Florida Universit\`aria\\
$^3$VRAIN - Valencian Research Institute for Artificial Intelligence, Universitat Polit\`ecnica de Val\`encia
\emails
javier.segovia@upf.edu,
yescudero@florida-uni.es,
serjice@dsic.upv.es 
}

\usepackage{marvosym}
\usepackage{fontawesome}
\usepackage{amsmath}
\usepackage{amssymb}
\usepackage{amsthm}
\usepackage{arydshln}
\usepackage{listings}
\usepackage{subcaption}
\usepackage{multirow}
\usepackage{multicol}
\usepackage{tikz}
\usetikzlibrary{arrows,automata}
\usetikzlibrary{shapes}
\newcommand{\tup}[1]{{\langle #1 \rangle}}
\usepackage{xcolor}
\usepackage{todonotes}
\usepackage{enumitem}

\usepackage{listings}
\newcommand{\strips}{\textsc{Strips}}     

\newtheorem{definition}{Definition}

\begin{document}

\maketitle

\begin{abstract}
The paper introduces a novel representation for {\em Generalized Planning} (GP) problems, and their solutions, as C++ programs. Our C++ representation allows to formally proving the termination of generalized plans, and to specifying their {\em asymptotic complexity} w.r.t. the number of world objects. Characterizing the complexity of C++ generalized plans enables the application of a combinatorial search that enumerates the space of possible GP solutions in order of complexity. Experimental results show that our implementation of this approach, which we call {\sc bfgp++}, outperforms the previous {\em GP as heuristic search} approach for the computation of generalized plans represented as compiler-styled programs. Last but not least, the execution of a C++ program on a classical planning instance is a deterministic grounding-free and search-free process, so our C++ representation allows us to automatically validate the computed solutions on large test instances of thousands of objects, where off-the-shelf classical planners get stuck either in the pre-processing or in the search.
\end{abstract}

\section{Introduction} 
{\em Automated planning} has not achieved the level of integration with common programming languages, like {\sc C}, {\sc Java}, or {\sc Python}, that is  achieved by other forms of problem solving such as {\em constraint satisfaction} or {\em operational research}~\cite{schulte2010modeling,prud2014choco,ortools}. An important  reason is the low-level representations traditionally handled in  planning~\cite{geffner2003pddl,rintanen2015impact}. Since the early 70's, \strips{} is the most popular representation language for research in automated planning~\cite{fikes1971strips}. Even today, \strips{} is an essential fragment of PDDL~\cite{haslum2019introduction}, the input language of the {\em International Planning Competition}, and most planners support the \strips{} features. In spite of its popularity, the \strips{} representation is too low-level for many interesting applications~\cite{smith2008anml,rintanen2015impact}; \strips{} limits  representation to Boolean state variables, and Boolean constraints, and focuses computation on plans represented as sequences of ground actions. 

Recent advances in {\em planning as heuristic search}~\cite{frances2017effective,say2017nonlinear,scala2020subgoaling}, and in {\em planning as satisfiability}~\cite{bryce2015smt,scala2016numeric,rintanen2017temporal}, show that handling more expressive problem representations does not necessarily increase planning complexity. On the other hand, advances in {\em generalized planning} (GP) are producing effective algorithmic solution representations for (possibly infinite) sets of planning instances that share common structure ~\cite{schmid2000applying,Winner03distill:learning,Levesque:GPlanning:IJCAI11,srivastava2011new,Zilberstein:Gplanning:icaps11,hu2011generalized,schmid2011inductive,belle2016foundations,illanes2019generalized,jimenez2019review,segovia2019computing,frances2021learning}. 

This paper introduces a novel C++ representation for GP. The contribution of the paper is three-fold:
\begin{enumerate}
    \item {\em Proving termination and characterizing complexity}. Our representation of GP solutions allows to formally proving the termination of generalized plans represented as C++ programs. In addition, our C++ representation reveals the {\em asymptotic complexity} of generalized plans w.r.t the number of world objects. This is a relevant topic beyond GP, since it allows defining formal upper-bounds on the complexity of the  (possibly  infinite)  set  of  instances of a classical planning domain.
    \item{\em Improving the GP as heuristic search approach}. By definition, any generalized plan built by {\sc bfgp++} is terminating. {\sc bfgp++} skips the costly check of infinite executions for candidate solutions and hence, it outperforms the previous {\em GP as heuristic search} approach~\cite{javi:GP:ICAPS21,segovia:GP:ijcai22}. 
    \item {\em  Validation of GP solutions at large instances}. The generalized plans produced by {\sc bfgp++} are compilable with standard programming tools, such as GCC {\em g++}, and efficiently validated in large instances (with thousands of objects), where off-the-shelf planners get stuck either in the pre-processing or in the search.
\end{enumerate}

\section{Preliminaries}

\subsection{Classical Planning}
Following the formalization by~\citeauthor{bonet2021general}~\citeyear{bonet2021general}, we define a {\em classical planning problem} as a pair $P=\tup{\mathcal{D},\mathcal{I}}$, where $\mathcal{D}$ is a first-order {\em planning domain} and $\mathcal{I}$ is the information of the classical planning {\em instance}.  The {\em domain} contains the set of predicate symbols $\Psi$ and the action schemes with preconditions and effects, given by atoms $p(x_1,\ldots, x_k)$, or their negations, where $p\in \Psi$ is a predicate symbol and each $x_i$ is a variable symbol representing an argument of the action scheme. A classical planning {\em instance} is a tuple $\mathcal{I}=\tup{\Omega,I,G}$ where $\Omega$ is the finite set of world objects. Last, $I$ and $G$ respectively are the {\em initial} and {\em goal} configurations of the world objects, and they are defined using ground atoms $p(o_1, \ldots, o_k)$, or their negation.  

The {\em set of states}, $S(P)$, associated with a classical planning problem $P$, are the possible sets of ground atoms. The initial state is $s_0=I$, and the subset of goal states $S_G\subseteq S(P)$, contains all the states $s_g\in S(P)$ s.t. $G\subseteq s_g$. The {\em state graph} associated with a classical planning problem $P$ has as nodes the states $S(P)$. Edges of this graph are defined as follows: for each pair of states $s\in S(P)$ and $s'\in S(P)$, the graph has a directed edge $(s,s')$ iff there is a ground action $a$ that is applicable in $s$ (i.e. whose preconditions hold in $s$) and whose effects transform the state $s$ into $s' = f(s,a)$ . 

A {\em solution} to a classical planning problem $P$ is a sequential plan $\pi = \tup{a_1,\ldots,a_m}$ such that $s_0 = I$,  the ground actions $a_i$ are applicable in states $s_{i-1}$, they produce successors $s_i = f(s_{i-1},a_i)$, and the goal condition holds in the last reached state, i.e. $G \subseteq s_m$.

\subsection{Generalized Planning}
This work builds on top of the inductive formalism for GP, where a GP problem is a set of classical planning instances that belong to the same domain $\mathcal{D}$. In other words, they are all represented with the same predicates and actions schemes, but they may differ in the number of objects, and  the initial/goal configuration of these objects. 

\begin{definition}[GP problem]
\label{def:gp-problem}
A {\em GP problem} is a non-empty set $\mathcal{P} =\{P_1,\ldots,P_T\}$ of $T$ classical planning instances from a given domain $\mathcal{D}$. 
\end{definition}

The aim of GP is to compute algorithmic planning solutions, a.k.a. generalized plans, that work for the given set of planning problems. In this paper we focus on the computation of GP solutions represented as C++ programs.

\begin{definition}[GP solution]
    \label{def:gp-solution}
    A {\em generalized plan} $\Pi$ solves a GP problem $\mathcal{P}=\{P_1,\ldots,P_T\}$ iff, for every classical planning instance $P_t\in \mathcal{P}$, $ 1\leq t\leq T$, the execution of $\Pi$ on $P_t$, denoted as $exec(\Pi,P_t) = \tup{a_1,\ldots,a_m}$, induces a classical plan that solves $P_t$.
\end{definition}

\section{A C++ Representation for  Planning}
We start explaining our C++ representation for propositional classical planning and then we show that it naturally extends to numeric planning and to GP. Our novel representation can actually be implemented with any structured programming language that supports {\tt If} conditionals and {\tt For} loops, as well as {\tt Vectors} (to store arrays that can change in size) and {\tt Map} containers (to store key-value pairs with unique keys); the paper exemplifies our representation with the {\em C++} language, but other common programming languages, such as {\em Python} or {\em Java}, could also be used.

\subsection{Classical planning problems as C++ programs}
We exemplify our C++ representation in the classic {\em blocksworld}~\cite{slaney2001blocks}, that consists of a set of blocks, a table, and a robot hand. The domain defines five first-order predicates, namely {\em clear(?x)}, {\em handempty()}, {\em holding(?x)}, {\em on(?x,?y)}, and {\em ontable(?x)}. A block that has nothing on it is clear, the robot hand can be empty or holding one block, and a block can be on top of another block or on the table. The domain defines four  action schemes, {\em stack(?x,?y)}, {\em unstack(?x,?y)}, {\em pick-up(?x)} and {\em put-down(?x)} for stacking (or unstacking) a block on top of another block and putting down (or picking up) a block onto the table.

{\bf States.}  Each first-order predicate  $p(x_1,\ldots, x_k)\in\Psi$ is represented as a C++ {\em map container} s.t. the key for indexing the map is a $k$-integer vector (where $k$ is the predicate arity). The map stores a Boolean for each of the corresponding ground predicates $p(o_1,\ldots, o_k)$ holding in the current state; 
we follow the closed-world assumption so our C++ state representation stores a maximum of $\sum_{k\geq 0} n_k|\Omega|^k$ Boolean, where $n_k$ is the number of first-order predicates with arity $k$, and $|\Omega|$ is the number of world objects\footnote{ $\sum_{k\geq 0} n_k|\Omega|^k$ is also the  number of propositions that result from grounding a \strips\ classical planning instance.}. Figure~\ref{fig:predicates} shows the C++ declaration of the blocksword predicates, while Figure~\ref{fig:tower} shows our C++ representation of a {\em blocksworld} state, where there are three blocks that are stacked in a single tower. 

\begin{figure}[t]
\centering
\begin{scriptsize}
\begin{lstlisting}[language=c++,numbers=none]
map<vector<int>, bool> pred_clear;
map<vector<int>, bool> pred_handempty;
map<vector<int>, bool> pred_holding;
map<vector<int>, bool> pred_on;
map<vector<int>, bool> pred_ontable;
\end{lstlisting}
\end{scriptsize}
    \caption{\small C++ declaration of the {\em blocksworld} first-order predicates.}
    \label{fig:predicates}
\end{figure}
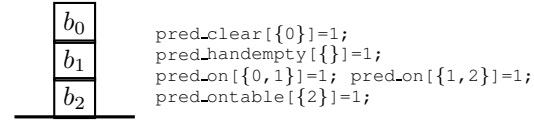
\begin{figure}[t]
\begin{minipage}{0.2\textwidth}
\begin{tikzpicture}[block/.style= {rectangle, draw=black, thick, text centered, node distance=.5cm}]
\node [block] (b1) [] {$b_0$};
\node [block] (b2) [below of= b1] {$b_1$};
\node [block] (b3) [below of=b2] {$b_2$};
\draw [very thick] (-.8,-1.25) -- (0.8,-1.25) node [] {};
\end{tikzpicture}
\end{minipage}
\hspace{-2cm}
\begin{minipage}[t]{0.7\textwidth}
\begin{scriptsize}
\begin{tabular}{l}
{\tt pred\_clear[\{0\}]=1;}\\
{\tt pred\_handempty[\{\}]=1;}\\
{\tt pred\_on[\{0,1\}]=1; pred\_on[\{1,2\}]=1;}\\
{\tt pred\_ontable[\{2\}]=1;}\\
\end{tabular}
\end{scriptsize}
\end{minipage}
    \caption{\small Example of a three-block  state from {\em blocksworld}  (left), and its corresponding  C++ representation (right).}   
    \label{fig:tower}
\end{figure}

{\bf Actions.} Each  action scheme is represented with a C++ Boolean function. The arguments of the function are those of the action scheme, but in our C++ representation they act as {\em indexes} to address the maps that are encoding the state. Formally, an {\em index} $z\in Z$ is a finite domain variable ranging the number of objects i.e., $D_z=[0,|\Omega|)$. We inductively define our C++ representation of an action scheme with the following grammar:

\begin{footnotesize}
\begin{align*}
Action :=\  & if(Condition_z(s))\{Effect(s)\}\ return\ false;\\
Condition_z(s) :=\ & (p(z_1,\ldots, z_k)==0) \&\&\ Condition_z(s)\mid \\
& (p(z_1,\ldots, z_k) ==1) \&\&\ Condition_z(s)\mid \\
& true\\
Effect(s) :=\ & (p(z_1,\ldots, z_k)=0);\ Effect(s)\mid \\
& (p(z_1,\ldots, z_k)=1);\ Effect(s)\mid \\
& return\ true;\\
\end{align*}
\end{footnotesize}
\noindent
where $Condition_z(s)$ is a conjunction of assertions over predicates $p(z_1,\ldots, z_k)$ instantiated with the action arguments, $==$ denotes the equality operator, $\&\&$ is the logical AND operator, $=$ indicates a value assignment, and $;$ denotes the end of an instruction. Likewise {\em Effect(s)} is a conjunction of assignments  representing the positive/negative effects of an action scheme;  $(p(z_1,\ldots, z_k) = 1)$ denotes a {\em positive} effect while $(p(z_1,\ldots, z_k) = 0)$
denotes a {\em negative} effect. 

Figure~\ref{fig:unstack2} shows our C++ representation of the {\tt unstack} action scheme from the blocksworld (Figure~\ref{fig:unstack1}), that compactly represents a set of state transitions, and that applies to any {\em blocksword} instance, no matter the number of blocks. Note that we represent actions as if they were always applicable, but action effects only update the  state iff the action preconditions hold in the current state. This action modeling, common in RL, facilitates the specification of compact algorithm-like solutions, and it preserves the original branching factor (successor states equal to their parents are ignored).

\begin{figure}[h!]
\begin{scriptsize}
\begin{lstlisting}[language=lisp,numbers=none]
(:action unstack
  :parameters (?x ?y)
  :precondition (and (on ?x ?y)(clear ?x)(handempty))
  :effect (and (holding ?x) (clear ?y)
               (not (clear ?x)) (not (handempty))
               (not (on ?x ?y)))))
\end{lstlisting}
\end{scriptsize}
    \caption{\small PDDL representation of the {\tt\small unstack} action scheme.}
    \label{fig:unstack1}
\end{figure}       
\begin{figure}[h!]
\begin{scriptsize}
\begin{lstlisting}[language=c++,numbers=none]
bool act_unstack(int x, int y){
  if(pred_on[{x,y}]==1  && pred_clear[{x}]==1  && pred_handempty[{}]==1){
    pred_holding[{x}] = 1; pred_clear[{y}] = 1;
    pred_clear[{x}] = 0; pred_handempty[{}] = 0;
    pred_on[{x,y}] = 0;
    return true;
  }
  return false;
}
\end{lstlisting}
\end{scriptsize}
    \caption{\small Our C++ representation of the {\tt\small unstack} action scheme.}
    \label{fig:unstack2}
\end{figure}

{\bf Problems.} Our C++ representation of a propositional planning problem is completed with the functions for representing the {\em initial state} and the {\em goals}. These functions are formalized as follows:
\begin{footnotesize}
\begin{align*}
Init :=\ & (p(o_1,\ldots, o_k)=1);\ Init\mid\\ 
&;\\
Goals :=\  & return(Condition_o(s));\\
Condition_o(s) :=\ & (p(o_1,\ldots, o_k)==0) \&\&\ Condition_o(s)\mid \\
& (p(o_1,\ldots, o_k)==1) \&\&\ Condition_o(s)\mid \\
& true
\end{align*}
\end{footnotesize}
\noindent
The {\em init function} is a write-only {\em void function} that initializes the C++ maps with the assignments for representing the  ground atoms $p(o_1,\ldots, o_k)$, that hold in the initial state. The {\em goal function} is a read-only {\em Boolean function} that encodes, as a partial state, the subset of goal states. Figure~\ref{fig:problem} shows the {\em init} and {\em goal} C++ functions for representing the problem of unstacking the three-block tower of Figure~\ref{fig:tower}. 

\begin{figure}[h!]
\begin{scriptsize}
\begin{lstlisting}[language=c++,numbers=none]
void init() {
 pred_clear[{0}]=1; 
 pred_handempty[{}]=1;
 pred_on[{0,1}]=1; pred_on[{1,2}]=1;
 pred_ontable[{2}]=1;
}
\end{lstlisting}
\begin{lstlisting}[language=c++,numbers=none]
bool goals() {
  return ((pred_ontable[{0}]==1) && 
          (pred_ontable[{1}]==1) && 
          (pred_ontable[{2}]==1));
}
\end{lstlisting}
\end{scriptsize}
    \caption{\small The {\em init} and {\em goal} C++ functions representing the planning problem of unstacking the three-block tower of Figure~\ref{fig:tower}.}
    \label{fig:problem}
\end{figure}

\subsection{Sequential plans as C++ programs} Our C++ representation of a sequential plan $\pi$ uses programming instructions for: (i), invoking the {\em C++ Boolean function} that encodes an action scheme and (ii), incrementing/decrementing the value of an index. Formally:
\begin{footnotesize}
\begin{align*}
\pi :=\ & Statement(s)\\
Statement(s) := \ & a(z_1,\ldots,z_k);Statement(s)\mid \\
\ & z++;Statement(s) \mid\\
\ & z--;Statement(s) \mid \\
\ & ; 
\end{align*}
\end{footnotesize}
\noindent
where $a(z_1,\ldots,z_k)$ is an action scheme instantiated with a subset of indexes $\{z_1,\ldots,z_k\}\subseteq Z$, and \{{\tt\small z++}, {\tt\small z--} {\small $| \; z \in Z\}$} are the instructions to increment/decrement an index $z\in Z$. Indexes in $Z$ are always initialized to zero. Figure~\ref{fig:indexes} illustrates the relation between an action scheme (i), its corresponding action instantiated over indexes in $Z$ and (ii), the corresponding ground actions instantiated over the objects in $\Omega$. 

Figure~\ref{fig:seq-plan-blocks} illustrates our C++ representation of the four-action sequential plan $\pi=\langle${\tt\small unstack(b0,b1)}, {\tt\small putdown(b0)}, {\tt\small unstack(b1,b2)}, {\tt\small putdown(b1)}$\rangle$ for unstacking the three-block tower of Figure~\ref{fig:tower}. The {\tt\small void ONTABLE-SEQUENTIAL()} program leverages two indexes, $Z=\{z_1,z_2\}$, that are initialized to zero so they initially point to the first object (block $b_0$ in this case). After executing the first $z_2${\tt\small ++} instruction, $z_2$ points to the second block, $b_1$, while $z_1$ still points to block $b_0$. This means that the first ${\tt\small act\_unstack(z_1,z_2)}$ instruction of the program in Figure~\ref{fig:seq-plan-blocks} is actually executing the ground action ${\tt\small unstack(b_0,b_1)}$, which corresponds to the first step of plan $\pi$. Likewise, the first ${\tt\small act\_putdown(z_1)}$ program instruction executes the ground action ${\tt\small putdown(b_0)}$, i.e. the second step of the sequential plan $\pi$. The second ${\tt\small act\_unstack(z_1,z_2)}$ program instruction is  executing the ground action ${\tt\small unstack(b_1,b_2)}$, since both $z_1$ and $z_2$ are increased just before that instruction is executed. Finally, the second ${\tt\small act\_putdown(z1)}$ executes the ground action ${\tt\small putdown(b_1)}$, which is the fourth and last step of the sequential plan $\pi$.

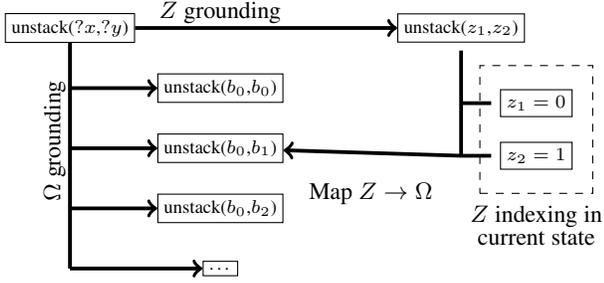
\begin{figure}[t]
    \centering
    \scriptsize
    \begin{tikzpicture}
    \node [draw] at (0, 0)  (a)    {unstack($?x$,$?y$)};
    \node [draw] at (5.2, 0) (b) {unstack($z_1$,$z_2$)};
    \node [draw] at (2, -0.8) (d) {unstack($b_0$,$b_0$)};
    \node [draw] at (2, -1.6) (e) {unstack($b_0$,$b_1$)};
    \node [draw] at (2, -2.4) (f) {unstack($b_0$,$b_2$)};
    \node [draw] at (2, -3.2) (g) {\ldots};    
    
    \draw[->, line width=0.5mm] (a) -- (b) node at (2,0.2) {\small $Z$ grounding};
    \draw[dashed] (5.45,-2.2) rectangle (6.95,-0.5);
    \node [draw] at (6.2,-1.0) {$z_1 = 0$};
    \node [draw] at (6.2,-1.7) {$z_2 = 1$};
    \node at (6.2,-2.5) {\small $Z$ indexing in};
    \node at (6.2,-2.8) {\small  current state};
    
    \path[line width=0.5mm] (a) edge (0, -3.2);
    \draw[->, line width=0.5mm] (0, -0.8) -- (d);
    \draw[->, line width=0.5mm] (0, -1.6) -- (e);
    \draw[->, line width=0.5mm] (0, -2.4) -- (f);
    \draw[->, line width=0.5mm] (0, -3.2) -- (g);    
    \node [draw=none, rotate=90] at (-0.2,-1.5) {\small $\Omega$ grounding};    
    
    \path[line width=0.5mm] (5.2,-0.3) edge (5.2,-1.7);
    \path[line width=0.5mm] (5.2,-1.0) edge (5.6,-1.0);
    \path[line width=0.5mm] (5.2,-1.7) edge (5.6,-1.7);
    \draw[->,line width=0.5mm] (5.2,-1.7) -- (e);
    \node at (4,-2.2) {\small Map $Z\rightarrow\Omega$};
    
    \end{tikzpicture}
    \caption{\small Relation between the action scheme $unstack(?x,?y)$ (i), the action $unstack(z_1,z_2)$ instantiated with indexes $(z_1,z_2)$, and (ii), the ground actions instantiated with the set of three blocks $\Omega=\{b_0,b_1, b_2\}$. Indexes $z_1$ and $z_2$ are bound variables in $[0,\ldots,|\Omega|)$ and currently, they are indexing blocks $b_0$ and $b_1$, respectively.}
    \label{fig:indexes}
\end{figure}

\begin{figure}[h!]
\begin{center}
\begin{scriptsize}
\begin{lstlisting}[language=c,numbers=none]
void ONTABLE-SEQUENTIAL (){
    int z1=0, z2=0;    
    z2++;
    act_unstack(z1,z2);
    act_putdown(z1);
    z1++;
    z2++;
    act_unstack(z1,z2);
    act_putdown(z1);
}
\end{lstlisting}
\end{scriptsize}
\end{center}
    \caption{\small C++ representation of a {\em sequential plan} for unstacking and putting onto the table the three-block tower of Figure~\ref{fig:tower}.}
    \label{fig:seq-plan-blocks}
\end{figure}

\begin{theorem}
Provided a number of indexes $|Z|$ as large as the largest arity of an action, our C++ representation for sequential plans preserves the original solution space.
\label{thm:programs-sound} 
\end{theorem}
\begin{proof}
Any sequential plan $\pi$ can be rewritten as an equivalent C++ procedure that first initializes indexes in $Z$ to zero and then, for each ground action $a\in \pi$, it repeatedly applies {\small\tt z++/z--} instructions until the indexes address the  objects of the corresponding ground action $a$. 
\end{proof}

\subsection{Beyond propositional planning}
Our C++ representation naturally supports planning with numeric state variables, that can participate into the representation of a classical planning problem as defined in PDDL2.1~\cite{fox2003pddl2}. To support the representation of numeric state variables, our {\em C++ maps} store {\em integers} instead of Boolean. Likewise goals and action preconditions can include assertions over the numeric state variables, and action effects can include assignments of the numeric state variables. For instance, the set of numeric state variables that indicate the physical distance between two blocks is declared in our C++ fragment as {\tt\small map< vector<int>, int> distance}; with this regard, $distance[\{0,1\}]=7$ indicates that the distance between blocks $b_0$ and $b_1$ is of seven units,  and $distance[\{z1,z2\}]>distance[\{z2,z3\}]$ indicates that the distance between the blocks pointed by indexes $z1$ and $z2$ is larger than the distance between the blocks pointed by $z2$ and $z3$. Object {\em typing} is also naturally supported by our C++ representation specializing map indexes to the number of objects of a particular type.

\section{Generalized Planning with C++}
This section details our approach for computing generalized plans represented as C++ programs.

\subsection{Representing generalized plans with C++}
\label{subsec:GPlanC++}
Given that a GP problem is just a set of classical planning problems from a given domain, and that the C++ representation for classical planning  is detailed above,  we directly define here our C++ representation of generalized plans. Our C++ representation of a generalized plan $\Pi$ extends our representation of a sequential plan $\pi$ with two control-flow constructs ({\tt If} conditionals and {\tt For} loops):
\begin{footnotesize}
\begin{align*}
\Pi :=\  & ExtdStmnt(s)\\
ExtdStmnt(s) :=\ & If; ExtdStmnt(s)\mid For; ExtdStmnt(s)  \mid\\
& Statement(s) ; ExtdStmnt(s)\mid ;\\
If :=\ & if(Condition)\{ExtdStmnt(s)\}\\
Condition := \ & (p(z_1,\ldots, z_k)==0)\mid (p(z_1,\ldots, z_k)\neq 0)\mid\\ 
& (z_1>z_2)\mid (z_1==z_2)\mid (z_1<z_2)\mid\\ 
& (p(z_1,\ldots, z_k)>p(z_1',\ldots, z_k'))\mid\\
& (p(z_1,\ldots, z_k)==p(z_1',\ldots, z_k'))\mid\\
& (p(z_1,\ldots, z_k)<p(z_1',\ldots, z_k'))\mid\\
For :=\ & for(z=0;z<|\Omega|;z++)\{ExtdStmnt(s)\}|\\
& for(z=|\Omega|-1;z\geq 0;z--)\{ExtdStmnt(s)\}
\end{align*}
\label{eq:GPgrammar}
\end{footnotesize}
\noindent
where {\tt\small ExtdStmnt(s)} extends {\tt\small Statement(s)}, as defined for sequential plans, with {\tt If} conditional and {\tt For} loop instructions. The {\tt\small Condition} of an {\tt If} instruction is restricted to: (i), checking whether a $p(z_1,\ldots, z_k)$ predicate instantiated with indexes in $Z$ equals to zero, (ii), the three {\em three-way comparison}~\cite{browning2020working} of two different indexes in $Z$ and (iii), the three-way comparison of two  predicates instantiated with indexes in $Z$ (or two numeric fluents in the case of a numeric domain). Last, we restrict {\tt For} loops to exclusively iterate over the domain $[0,|\Omega|)$ of an index $z\in Z$. Like in our C++ representation for sequential plans, in a C++ generalized plan the set of indexes $Z$ are always initialized to zero.

Figure~\ref{fig:GP-blocks} shows an example of a generalized plan, represented as a C++ program and computed by {\sc bfgp++}, for unstacking any number of towers from the blocksworld, no matter the actual number of blocks $|\Omega|$. Please note that object ordering affects to the  sequential plan produced by the execution of the generalized plan but it does not affect the correctness/completeness of the generalized plan\footnote{This also occurs with regular programs e.g. a {\em SelectionSort} program is sound and complete but the number of {\tt swap} instructions executed by the program depends on the input list to be sorted.}. As a matter of fact, the program of Figure~\ref{fig:GP-blocks} leverages three indexes $Z=\{z_1,z_2,z_3\}$ to be robust to any  block ordering. 

\begin{figure}[t]
\begin{scriptsize}
\begin{lstlisting}[language=c++,numbers=none,mathescape]
void ONTABLE (){           
 int z1=0, z2=0, z3=0;     
 for(z1=0; z1<$|\Omega|$; z1++){     
  for(z2=0; z2<$|\Omega|$; z2++){     
   for(z3=0; z3<$|\Omega|$; z3++){     
    act_put_down(z2);          
    act_unstack(z2,z3);       
   }}}}
\end{lstlisting}
\end{scriptsize}
    \caption{\small Generalized plan, represented as a C++ program and computed by {\sc bfgp++}, for unstacking any number of towers from the blocksworld, no matter the number of blocks $|\Omega|$.}
    \label{fig:GP-blocks}
\end{figure}

\subsubsection{From C++ programs to sequential plans}
Given a C++ generalized plan $\Pi$, and a classical planning problem $P$, the sequential plan $exec(\Pi,P)$ is built executing $\Pi$ on $P$. This is a deterministic search-free procedure where no instantiation is required (there are no {\em free variables}). The program execution may however produce actions whose execution does not modify the planning state, i.e. the current state does not meet the precondition of those actions, so action effects are not applied. These actions are automatically discarded since identifying them is straightforward; the execution of the corresponding C++ function $a(z_1,\ldots,z_k)$ returns {\tt false}. 

We illustrate the building of a sequential plan from a C++ generalized plan with the execution of the program of Figure~\ref{fig:GP-blocks} on the initial state of Figure~\ref{fig:tower}. This execution produces the following sequence of ground actions repeated thrice: $3\times\langle${\small putdown(b0), unstack(b0,b0), putdown(b0),} {\footnotesize\bf unstack(b0,b1)}, {\footnotesize\bf putdown(b0)}, {\small unstack(b0,b2),} {\footnotesize\bf putdown(b1)}, {\small unstack(b1,b0), putdown(b1), unstack(b1,b1), putdown(b1),} {\footnotesize\bf unstack(b1,b2)}, {\small putdown(b2), unstack(b2,b0), putdown(b2), unstack(b2,b1), putdown(b2), unstack(b2,b2)}$\rangle$. Only the four actions in bold update the state, i.e. they are applicable when they are executed; in the first repetition {\small unstack(b0,b1)}, {\small putdown(b0)} and {\small unstack(b1,b2)}, are applicable; while {\small putdown(b1)}, becomes only applicable in the second out of the three repetitions. The sequence of these four ground actions actually is the sequential plan for unstacking the three-block tower of Figure~\ref{fig:tower}. 

\subsubsection{Termination and complexity of generalized plans}
Enabling C++ programs with {\tt For} loops introduces a new possible source of failure of a program execution in a given classical planning instance; the program execution could enter into an {\em infinite loop}. To formally guarantee the termination of our C++ generalized plans, we restrict ourselves to C++ programs s.t. {\em ExtdStmnt(s)} in the body of a {\tt For} loop do not include instructions that modify the index iterated by that  loop~\cite{hoare1969axiomatic}.

\begin{theorem} 
A generalized plan $\Pi$, represented in our C++ fragment is always terminating.
\label{thm:programs-terminating} 
\end{theorem}
\begin{proof}
The grammar used for the production of a C++ generalized plan $\Pi$ comprises sequential statements, {\tt If} conditionals, and {\tt For} loops that iterate over the finite  domain $[0,|\Omega|)$ of an index. Since these rules guarantee the program is well structured, and given that sequential statements and {\tt If} conditionals can only advance  the program execution, then {\tt For} loops would be the only possible cause of non-termination. Since we do not allow to programming instructions inside a {\tt For} loop that modify the index iterated by the loop, we have that (nested) {\tt For} loops always modify indexes in a unique direction, which is terminating because the domain of an index is finite by definition. 
\end{proof}

Besides providing a formal termination proof for C++ generalized plans, restricting loops to iterations over the domain $[0,|\Omega|)$ of an index also reveals the {\em asymptotic complexity} of generalized plans w.r.t.~the number of objects. We leverage the {\em big-O notation} to formulate an upper-bound on the complexity of generalized plans and, as a consequence, on the difficulty of the problems solved by those plans~\cite{skiena1998algorithm}. If there is a C++ generalized plan $\Pi$ that solves a given classical planning problem $P$, it means that the length of the sequential plan produced by the execution of $\Pi$ on $P$ is upper-bounded by a polynomial of the number of objects $|\Omega|$; the degree of that polynomial is given by the maximum number of nested {\em for loops} in the C++ program. As $|\Omega|$ grows larger this term will come to dominate, so that all other terms, and the coefficients, can be ignored.

\begin{definition}[Asymptotic complexity of a generalized plan]
\label{def:complexity}
The asymptotic complexity of a generalized plan represented by a C++ program is defined as the joint product of the ranges of the indexes iterated by the largest set of nested loops.
\end{definition}

For instance the generalized plan of Figure~\ref{fig:GP-blocks}, with three nested loops, and where each loop range is $[0,|\Omega|)$, has asymptotic complexity $O(|\Omega|^3)$. This is a worst-case upper-bound; the program of Figure~\ref{fig:GP-blocks} exhibits worst-case complexity $O(|\Omega|^3)$ but we showed that the produced sequential plan for the three-block instance of Figure~\ref{fig:tower} contained four actions.

\subsection{Synthesis of C++ programs as heuristic search}
{\sc bfgp++} is our approach to the synthesis of generalized plans, represented as C++ programs. {\sc bfgp++} implements a Best-First Search (BFS) in the space of  possible programs that can be built with the grammar of Section~\ref{subsec:GPlanC++}. Since this  search space is  unbound, we bound it with two input parameters: a maximum number program lines $n$, and a maximum number of {\em indexes} $|Z|$ that can be used by a program. Next, we provide more details on the implementation of {\sc bfgp++}.

\subsubsection{The search space} Each node of the {\sc bfgp++} search space corresponds to a {\em partially specified program}. By partially specified program we mean that some of its $n$ program lines may be undefined, because they are not programmed yet. Starting from the {\em empty program} (i.e. the partially specified program whose $n$ lines are all undefined), we  enumerate the space of  possible programs with a search operator that programs a possible programming instruction (according to the grammar of Section~\ref{subsec:GPlanC++}), at an undefined program line $0\leq i< n$. This search operator is only applicable when program line $i$ is undefined. Initially $i:=0$, and after line $i$ is programmed $i:=i+1$. For instance, we can build the generalized plan of Figure~\ref{fig:GP-blocks} with the index set $Z=\{z_1,z_2,z_3\}$, starting from the empty program, and following the successive application of the following six grammar rules (indexes in $Z$ are always initialized to zero): 
\begin{scriptsize}
\begin{enumerate}
    \item $for(z1=0; z1<|\Omega|; z1{\small ++})\{ExtdStmnt(s)\}$
    \item $for(z2=0; z2<|\Omega|; z2{\small ++})\{ExtdStmnt(s)\}$
    \item $for(z3=0; z3<|\Omega|; z3{\small ++})\{ExtdStmnt(s)\}$
    \item $act\_putdown(z2); ExtdStmnt(s)$
    \item $act\_unstack(z2,z3); ExtdStmnt(s)$
    \item $;$
\end{enumerate}
\end{scriptsize}

\subsubsection{Pruning rules} To keep the search space tractable, {\sc bfgp++} implements the following pruning rules that reduce the search space but preserve the solution space:
\begin{itemize}
    \item We do not allow  programming a {\em conditional if} or a {\em for loop} at the last program line; it is meaningless to program control-flow structures with empty body. 
    \item We do not allow  programming {\em for loops} that iterate over an object type that contains, at most, a single object for all the instances of the GP problem.  
\end{itemize}

In addition, we implement a simple but effective symmetry breaking mechanism to safely prune the programming of  {\em for loops} that  correspond to {\em symmetries} (permutations of the order of the {\em for loops}) of already expanded programs.  We use a {\em tabu list}, that is initially empty and, every time a node is expanded, we store an abstraction of its programmed {\em for loops}; the remaining program instructions, and the precise identity of the indexes iterated by the {\em for loops} are ignored. For instance, the generalized plan of Figure~\ref{fig:GP-blocks} is abstracted by \{(block,++,1,8), (block,++,2,7), (block,++,3,6)\}, since this program contains three {\em for loops} and where, {\tt block} indicates the index type, {\tt ++} indicates the kind of the {\em for loop} ({\tt\small ++}increasing/{\tt\small--}decreasing), and each pair of numbers indicates the first and  last program lines of each {\em for loop}.

\subsubsection{The evaluation functions} {\sc bfgp++} implements a {\em Best-First Search} (BFS) in the previous solution space. To reduce memory requirements, we store only the open list of generated nodes but not the closed list of expanded nodes~\cite{korf2005frontier}. We consider the following evaluation functions for sorting the open list:
        \begin{itemize}
            \item $f_{euclidean}(\Pi,\mathcal{P}) = \sum_{P_t\in\mathcal{P}}\sum_{v\in G_t} (s_t[v]-G_t[v])^2$. This function accumulates, for each classical planning problem $P_t\in\mathcal{P}$ in a GP problem, the {\em euclidean distance} of  state $s_t$ to the goal state variables $G_t$. The state $s_t$ is obtained applying the sequence of actions  $exec(\Pi,P_t)$ to the initial state $I_t$ of that problem $P_t\in\mathcal{P}$.
            \item $f_{min(\#loops)}(\Pi)$ is the number of loop instructions in $\Pi$.    
        \end{itemize}
\noindent

Both $f_{euclidean}$ and $f_{min(\#loops)}$  were proposed in \citeauthor{javi:GP:ICAPS21}~(\citeyear{javi:GP:ICAPS21}), named $h_5$ and $f_1$ respectively. In more detail, the configuration with the best reported performance was $(f_{euclidean},f_{min(\#loops)})$, i.e. sorting the open list with $f_{euclidean}$, and breaking ties with the $f_{min(\#loops)}$ function. In this paper we introduce a third alternative function $f_{max(\#loops)}(\Pi)=-f_{min(\#loops)}(\Pi)$, which aims maximizing the number of loops in a program. We discovered that prioritizing by $f_{euclidean}$ and breaking ties with our new function $f_{max(\#loops)}$, instead of $f_{min(\#loops)}$, performs much better at a wide range of challenging GP problems. Section~\ref{sec:evaluation} provides  details on the obtained results. 

\subsubsection{Properties of {\sc bfgp++}}
Our {\sc bfgp++} algorithm for the synthesis of C++ generalized plans is terminating; termination follows from a terminating searching algorithm and evaluation functions. Regarding the former, {\sc bfgp++} implements a {\em frontier BFS} which is known to be terminating at finite search spaces~\cite{korf2005frontier}, we recall that the search space of {\sc bfgp++} is finite since $n$ and $|Z|$ are bounded. Regarding the evaluation functions, $f_{min(\#loops)}$ and $f_{max(\#loops)}$, they terminate in $n$ steps, where $n$ is the number of lines of the program to evaluate, and $f_{euclidean}$ terminates iff the program executions terminate, which immediately follows from Theorem~\ref{thm:programs-terminating}. Further, {\sc bfgp++} is sound, since it only outputs a program when it is able to solve all the planning instances given as input in the GP problem (Definition~\ref{def:gp-solution}). Last, {\sc bfgp++} is complete provided that there exists a GP solution within the given number of program lines $n$ and indexes $|Z|$.

\begin{table*}[t]
\scriptsize    
\centering
    \begin{tabular}{|l|c||c|c|c|c||c|c|c|c|} \hline
         \multirow{2}{*}{} & \multirow{2}{*}{$n, |Z|$} & \multicolumn{4}{c||}{GP as heuristic search (\citeyear{javi:GP:ICAPS21})} & \multicolumn{4}{c|}{\sc bfgp++} \\\cline{3-10}
         & & Time & Mem. & Exp. & Eval.  & Time & Mem. & Exp. & Eval.  \\\hline
         Blocks (ontable) & 9, 3 & TE & TE & TE & TE & {\bf 0.08} & {\bf 5} & {\bf 9} & {\bf 347} \\         
         Corridor & 11, 2 & TE & TE & TE & TE & {\bf 701.90} & {\bf 27} & {\bf 661.4K} & {\bf 2.5M} \\
         Fibonacci & 7, 2 & 32.05 & 83 & 242.1K & 1.7M & {\bf 1.61} & {\bf 5} & {\bf 2.4K} & {\bf 19.1K} \\
         Find & 6, 3 & 0.98 & 7 & 13.9K & 38.4K & {\bf 0.14} & {\bf 5} & {\bf 1.3K} & {\bf 3.9K} \\
         Floyd & 8, 3 & TE & TE & TE & TE & {\bf 0.22} & {\bf 5} & {\bf 4} & {\bf 138} \\
         Gripper & 8, 4 & TE & TE & TE & TE & {\bf 105.00} & {\bf 206} & {\bf 83.2K} & {\bf 1.0M} \\
         Intrusion & 9, 1 & TE & TE & TE & TE & {\bf 521.24} & {\bf 874} & {\bf 411.8K} & {\bf 3.4M} \\
         Reverse & 7, 2 & 9.75 & 31 & 99.8K & 344.7K & {\bf 0.23} & {\bf 4} & {\bf 626} & {\bf 2.8K} \\
         Select & 7, 2 & 9.52 & 32 & 96.3K & 331.6K & {\bf 0.28} & {\bf 4} & {\bf 737} & {\bf 4.4K} \\
         Sorting & 8, 2 & 129.72 & 335 & 1.2M & 4.1M & {\bf 0.02} & {\bf 4} & {\bf 52} & {\bf 245} \\
         Spanner & 12, 5 & TE & TE & TE & TE & {\bf 0.91} & {\bf 5} & {\bf 14} & {\bf 367} \\
         Triangular Sum & 5, 2 & 0.15 & 5 & 1.4K & 9.9K & {\bf 0.01} & {\bf 4} & {\bf 9} & {\bf 96} \\
         Visitall & 15, 4 & TE & TE & TE & TE & {\bf 1.67} & {\bf 6} & {\bf 117} & {\bf 2.7K} \\\hline
    \end{tabular}
    \caption{\small Number $n$ of program lines and $|Z|$ indexes. For each synthesis configuration we report CPU time (secs), memory peak (MBs), and number of expanded and evaluated nodes. Best results in bold. TE stands for Time-Exceeded ($>$1h of CPU). }
    \label{tab:synthesis}
\end{table*}

\section{Evaluation}
\label{sec:evaluation}
Next we report results on the synthesis (and validation) of C++ generalized plans with {\sc bfgp++}, and compared them w.r.t. the {\em GP as heuristic search} approach~\cite{javi:GP:ICAPS21,segovia:GP:ijcai22}, the state-of-the-art for the computation of generalized plans represented as programs. All experiments  are performed in an Ubuntu 20.04 LTS, with AMD® Ryzen 7 3700x 8-core processor $\times$ 16 and 32GB of RAM, with a 1 hour time bound. We also report the computed solutions, and their asymptotic complexity w.r.t the number of world objects. 

We address the full set of domains from \citeauthor{javi:GP:ICAPS21}~(\citeyear{javi:GP:ICAPS21}), all numeric domains, and we also address several \strips\ domains (marked with a {\bf *}). Next we briefly describe the domains: {\bf Blocks* (ontable)}, put all blocks onto the table from any number of towers and blocks setting. {\bf Corridor*}, move from any initial location in a corridor to any other arbitrary location. {\bf Fibonacci}, given the first numbers of the Fibonacci sequence, compute the $n^{th}$ value. {\bf Find}, count the number of occurrences of a specific value in a vector. {\bf Floyd*}, given an input graph, compute a path that connects two distant nodes. {\bf Gripper*}, move all balls from room A to the adjacent room B. {\bf Intrusion*}, steal data from a set of hosts. {\bf Reverse}, reverse the content of a given vector. {\bf Select}, search for the minimum value in a vector. {\bf Sorting}, sort in increasing order all the values in a vector. {\bf Spanner*}, collect the spanners in a directed corridor to tighten all loose nuts. {\bf Triangular Sum}, given the first numbers of the triangular sum, compute the $n^{th}$ value. {\bf Visitall*}, starting at the bottom left corner of a squared grid visit all locations in the grid.

\subsubsection{Synthesis of C++ generalized plans}
For the synthesis experiments, we use ten random instances of increasing difficulty per domain and compare {\sc bfgp++}, that leverages the function combination $(f_{euclidean},f_{max(\#loops)})$, with the best configuration of the {\em GP as heuristic search} approach~\cite{javi:GP:ICAPS21,segovia:GP:ijcai22}, that leverages the evaluation functions $(f_{euclidean},f_{min(\#loops)})$ and serves as a baseline. 

Table~\ref{tab:synthesis} summarizes the obtained results, when computing the best solutions found in terms of number of required program lines and pointers, and it shows that {\sc bfgp++} outperforms the baseline over all the domains. One of the main performance issues in the synthesis of planning programs is to guarantee that a given (partially specified) program is terminating~\cite{segovia:GP:AAAI2020}; in \citeauthor{javi:GP:ICAPS21}~(\citeyear{javi:GP:ICAPS21}), search nodes corresponding to infinite programs were checked and discarded in execution time, which was costly and had no guarantees beyond the given set of input instances. Alternatively, the candidate C++ generalized plans considered by {\sc bfgp++} are, by definition, terminating for any input instance. Therefore {\sc bfgp++} skips the costly check of infinite programs. 

Note that the baseline fails to solve  the \strips\ domains within the given computation bounds. In \strips\ domains, $f_{euclidean}$ is a simple {\em goal counting} heuristic. The function combination of {\sc bfgp++} $f_{euclidean},f_{max(\#loops)}$  breaks ties prioritizing C++ programs of higher asymptotic complexity. This means that, if a C++ generalized plan does not exist within $n-1$ lines, but it does exist for $n$ lines, it will probably contain as many nested loops as possible. This is not however a {\em rule of thumb}, as shown by the {\em Corridor} domain, which is the domain that took the largest time to be solved; requiring two consecutive (not nested) {\tt For} loops, which are hard to identify by our current evaluation functions.

\subsubsection{Validation of C++ generalized plans}
The generalized plans synthesized by {\sc bfgp++} are compilable with the {\em GCC g++} compiler, and they are all successfully validated on twenty large random instances of increasing size. Table~\ref{tab:validation_experiments} reports the size of the largest
instance, and the average and total validation times (including compilation time).  This experiment shows that representing generalized plans as C++ programs is a scalable approach to deal with large classical planning instances, of thousands of objects, where off-the-shelf classical planners get stuck even in the pre-processing. 

Next we report the computed programs and their asymptotic complexity. Please note that the reported validation times correlate to the revealed program complexity.

{\bf Floyd.} Indexes  increase up to $N=|\Omega|$, which denotes the number of graph nodes, of the input instances. The asymptotic complexity of the generalized plan is $O(N^3)$.

\begin{scriptsize}
\begin{lstlisting}[language=c++,numbers=none,mathescape]
void FLOYD(){
 int z1=0, z2=0, z3=0;
 for(z1=0; z1<N; z1++){
  for(z2=0; z2<N; z2++){
   for(z3=0; z3<N; z3++){
    act_connect(z1,z2,z3);
}}}}
\end{lstlisting}
\end{scriptsize}

{\bf Corridor.}  This generalized plan requires the composition of two {\em for loops}, the first one to move to the rightmost location, and a second one to move back until reaching the goal location. In the worst case it iterates twice over the set of locations $N$, and its asymptotic complexity is $O(N)$.

\begin{scriptsize}
\begin{lstlisting}[language=c++,numbers=none,mathescape]    
void CORRIDOR(){
 int l1 = 0, l2 = 0; 
 for(l1=0; l1<N; l1++){
   act_move(l2,l1);
   l2 = l1;
 }
 for(l1=N-1; l1>=0; l1--){
   if(pred_goal_at[{l1}] == 0 ){
     if (l2>0) l2--;
   }
   act_move(l1,l2);
 }} 
\end{lstlisting} 
\end{scriptsize}

{\bf Fibonacci and Triangular Sum.} Both are numeric domains, where an input vector is given with the first elements of the sequence to compute. They compute every $a$-th value iteratively until reaching the last $N$ value. Their asymptotic complexity is $O(N)$.

\begin{scriptsize}
\begin{lstlisting}[language=c++,numbers=none,mathescape]    
void FIBONACCI(){      void TRIANGULARSUM(){
 int a=0, b=0;           int a=0, b=0;
 for(a=0; a<N; a++){     for(a=0; a<N; a++){
  act_vector_add(a,b);    act_vector_add(a,b);
  if (b>0) b--;           b = a;
  act_vector_add(a,b);   }}
  b = a;}}
\end{lstlisting}
\end{scriptsize}

{\bf Find, Reverse, Select and Sorting.} Four numeric domains for vector manipulation. $N$ is the size of the input vector, the first three domains run in $O(N)$. {\em Sorting} has two nested loops so its complexity is $O(N^2)$.
\begin{scriptsize}
\hspace{-5cm}
\begin{lstlisting}[language=c++,numbers=none,mathescape]        
void FIND(){             void SELECT(){
 int i=0, t=0, a=0;       int a=0, b=0;
 for(i=0; i<N; i++){      for(a=0; a<N; a++){
  if(p_v[{i}]==p_v[{t}]){  if(p_v[{a}]<p_v[{b}]){
     act_accumulate(a);     b=a; }}
  }}}                     act_select(b);}

void REVERSE(){          void SORTING(){
 int i=0, j=0;            int i=0, j=0;
 for(i=N-1; i>=0; i--){   for(i=0; i<N; i++){
  if(i>j){                 for(j=0; j<N; j++){
   swap(p_v[{i}],p_v[{j}]); if(p_v[{i}]<p_v[{j}]){
  }                          swap(p_v[{i}],p_v[{j}]);}}}}
  if(j<N-1) j++; }}           
                                
\end{lstlisting}
\end{scriptsize}

{\bf Gripper and Intrusion.} These two domains have generalized plans where a sequence of planning actions is repeated over the finite set of objects: the number of balls $NB$ for {\em Gripper}, and the number of hosts $NH$ for {\em Intrusion}. Their complexity is $O(NB)$ and $O(NH)$, respectively.

\begin{scriptsize}
\begin{lstlisting}[language=c++,numbers=none,mathescape]   
void GRIPPER(){            void INTRUSION(){
 int r1=0, r2=0, b1=0, g1=0; int h1=0;
 for(b1=0; b1<NB; b1++){     for(h1=0; h1<NH; h1++){
   act_pick(b1,r1,g1);        act_recon(h1);
   if (r2<NROOMS-1) r2++;     act_break_into(h1);
   act_move(r1,r2);           act_clean(h1);
   act_drop(b1,r2,g1);        act_gain_root(h1);
   act_move(r2,r1); }}        act_download_files(h1);
                              act_steal_data(h1);  }}
\end{lstlisting}
\end{scriptsize}

{\bf Spanner.} The complexity of this generalized plan is  $O(NLOC^2 \cdot NNUTS \cdot NSPAN)$, where $NLOC$ is the number of  locations, $NNUTS$ is the number of nuts, and $NSPAN$ is the number of spanners.

\begin{scriptsize}
\begin{lstlisting}[language=c++,numbers=none,mathescape]    
void SPANNER(){
 int l1=0, l2=0, m1=0, n1=0, s1=0; 
 for(l1=0; l1<NLOC; l1++){
  for(l2=0; l2<NLOC; l2++){
   for(n1=0; n1<NNUTS; n1++){
    for(s1=0; s1<NSPAN; s1++){
     act_pickup_spanner(l1,s1,m1);
     act_tighten_nut(l1,s1,m1,n1);
     act_walk(l1,l2,m1); }}}}
\end{lstlisting}
\end{scriptsize}

{\bf Visitall.} The complexity of this generalized plan is $O(NROWS^2 \cdot NCOLS^2)$, where $NROWS$ is the number of grid rows and $NCOLS$ the number of grid columns. 

\begin{scriptsize}
\begin{lstlisting}[language=c++,numbers=none,mathescape]    
void VISITALL(){
 int r1=0, r2=0, y1=0, y2=0; 
 for(r1=0; r1<NROWS; r1++){
  for(r2=0; r2<NROWS; r2++){
   for(c1=0; c1<NCOLS; c1++){
    for(c2=NCOLS-1; c2>=0; c2--){
     act_right(c1,c2);
     act_visit(r1,c2);
     if(p_at_row[{r2}]==0){
      act_up(r1,r2);
      act_left(c2,c1); }}}}}}
\end{lstlisting}
\end{scriptsize}

\begin{table}[]
    \centering
    \begin{scriptsize}
    \begin{tabular}{|l||c|c|c|} \hline
         {\bf } & {\bf Max. input} & {\bf Avg. instance time (s)} & {\bf Total time (s)  }\\\hline 
         Blocks (ont.) & 1,000 blocks & 33.592$\pm$37.888 & 673.427 \\
         Corridor & 5,001 locs. & 0.010$\pm$ 0.005 & 1.495 \\
         Fibonacci & 44th num. & 0.001$\pm$0.000 & 1.285 \\
         Find & 5,001 elems. & 0.004$\pm$0.002 & 1.352 \\
         Floyd & 872 vertices & 30.502$\pm$38.561 & 611.308 \\
         Gripper & 5,001 balls &  0.011$\pm$ 0.005 & 1.702 \\
         Intrusion & 1,001 hosts & 0.003$\pm$0.001 & 1.568\\
         Reverse & 5,001 elems. & 0.007$\pm$0.003 & 1.416 \\
         Select & 5,001 elems. & 0.010$\pm$0.004 & 1.485 \\
         Sorting & 5,001 elems. & 1.337$\pm$1.200 & 28.029  \\
         Spanner & 212 spanners & 20.609$\pm$26.363 & 413.798 \\
         T. Sum & 5,001st num. &  0.008$\pm$0.004 & 1.458 \\
         Visitall & $100\times100$ grid & 6.738$\pm$7.474 & 136.325 \\\hline 
    \end{tabular}
\end{scriptsize}    
    \caption{\small Size of the largest instance, avg. instance time with standard deviation and total validation time including compilation time (both in seconds). }
    \label{tab:validation_experiments}
\end{table}

\section{Conclusions} 
We presented a novel C++ representation for GP problems, and their solutions. We also introduced {\sc bfgp++} that implements a heuristic search in the space of candidate C++ generalized plans, which naturally models \strips\ domains, and that outperforms the previous {\em GP as heuristic search} approach~\cite{javi:GP:ICAPS21,segovia:GP:ijcai22}; since our C++ programs are terminating by definition, {\sc bfgp++} can skip the costly check of infinite programs. In addition we showed that our C++ generalized plans are compilable with standard programming tools (GCC g++) and that they are successfully validated in large instances, with several thousands of objects, where off-the-shelf classical planners get stuck in the pre-processing. Our {\em cost-to-go heuristic} is less informed than current heuristics for classical planning; our research agenda is to obtain better estimates building on top of modern  heuristics for classical planning~\cite{segovia2022landmarks}

Our C++ representation for GP can be viewed as an instantiation of {\em F-\strips{}}~\cite{geffner2000functional}; we show that the single level of indirection of indexes  over objects  is enough to represent GP problems, and their solutions, with constant memory access; in {\em F-\strips{}} function symbols can be recursively nested, so the actual complexity of state queries in {\em F-\strips{}} depends on the depth of this nesting. Prior work characterized the complexity of planning tasks following a related but different approach; analysing the behaviour of {\em state-space heuristics} for classical planning~\cite{hoffmann2005ignoring,helmert2008understanding,seipp2016correlation}. When generalized plans are represented as {\em generalized policies} their complexity has also been  characterized  w.r.t. the features in the policy rules~\cite{bonet2021general}.

\bibliographystyle{named}
\bibliography{ijcai22}

\end{document}